\documentclass[a4paper,fleqn]{cas-sc}

\usepackage[authoryear,longnamesfirst]{natbib}
\usepackage{subcaption}
\usepackage[T1]{fontenc}
\usepackage{algorithm}
\usepackage{algpseudocode}
\usepackage{setspace}

\def\tsc#1{\csdef{#1}{\textsc{\lowercase{#1}}\xspace}}
\tsc{WGM}
\tsc{QE}

\newtheorem{theorem}{Theorem}
\newproof{pot}{Proof of Theorem \ref{thm}}
\newtheorem{definition}{Definition}%

\newcommand{\hlColor}{black}

\begin{document}

\doublespacing

\let\WriteBookmarks\relax
\def\floatpagepagefraction{1}
\def\textpagefraction{.001}

\shorttitle{Explainable Classification of Astronomical Uncertain Time Series}    

\shortauthors{Mbouopda et al.}  

\title [mode = title]{Explainable Classification of Astronomical Uncertain Time Series}  



%

\author[1]{Michael Franklin Mbouopda}[orcid=0000-0003-1916-1464]

\cormark[1]


\ead{michael_franklin.mbouopda@doctorant.uca.fr}

\ead[url]{https://frankl1.github.io}

\credit{Conceptualization of this study, Data understanding, Methodology, Experiment, Result analysis, Software, Writing}

\affiliation[1]{organization={University Clermont Auvergne, Clermont Auvergne INP, ENSMSE, CNRS, LIMOS},
           city={Clermont–Ferrand},
            postcode={63000}, 
            country={France}}

\author[2]{Emille E. O. Ishida}[orcid=0000-0002-0406-076X]


\ead{emille.ishida@clermont.in2p3.fr}

\ead[url]{https://www.emilleishida.com/}

\credit{Data understanding, result analysis, writing}

\affiliation[2]{organization={University Clermont Auvergne, CNRS/IN2P3, LPCA},
             city={Clermont-Ferrand},
            postcode={63000}, 
            country={France}}

\author[1]{Engelbert {Mephu Nguifo}}[orcid=0000-0001-9119-678X]


\ead{engelbert.mephu_nguifo@ucaf.fr}

\ead[url]{https://perso.isima.fr/~enmephun/}

\credit{Conceptualization of this study, Data understanding, Result analysis, Writing}


\author[2]{Emmanuel Gangler}[orcid=0000-0001-6728-1423]


\ead{emmanuel.gangler@clermont.in2p3.fr}

\ead[url]{https://annuaire.in2p3.fr/2113-2297/emmanuel-gangler}

\credit{Data preparation, Data understanding, Result analysis, Writing}

\cortext[1]{Corresponding author}


\nonumnote{This work is funded by the French Ministry of Higher Education, Research and Innovation. Thanks to the TransiXplore project team which helped us understanding and preprocessing the PLAsTiCC dataset and thanks to the Large Synoptic Survey Telescope (LSST) project which published the dataset. Thanks to the anonymous reviewers for their constructive remarks. We are also grateful to authors of public softwares and data sets used in our experimentation.}

\begin{abstract}
Exploring the expansion history of the universe, understanding its evolutionary stages, and predicting its future evolution are important goals in astrophysics. Today, machine learning tools are used to help achieving these goals by analyzing transient sources, which are modeled as uncertain time series. Although \textit{black-box} methods achieve appreciable performance, existing interpretable time series methods failed to obtain acceptable performance for this type of data. Furthermore, data uncertainty is rarely taken into account in these methods. In this work, we propose an uncertainty-aware subsequence based model which achieves a classification comparable to that of state-of-the-art methods. Unlike conformal learning which estimates model uncertainty on predictions, our method takes data uncertainty as additional input. Moreover, our approach is explainable-by-design, giving  domain experts the ability to inspect the model and explain its predictions. The explainability of the proposed method has also the potential to inspire new developments in theoretical astrophysics modeling by suggesting important subsequences which depict details of light curve shapes. The dataset, the source code of our experiment, and the results are made available on a public repository.
\end{abstract}


\begin{highlights}
\item We proposed an accurate and explainable-by-design method for uncertain time series classification.
\item We showed the effectiveness and trustworthiness of our method on a real uncertain time series dataset from the astrophysics domain. 
\item We open-sourced the data and the code used in our study.
\end{highlights}


\maketitle

\doublespacing

\section{Introduction}
Machine learning (ML) has become an ineluctable tool for analyzing and extracting meaningful information from data. Classically exclusively applied on tabular data, it is nowadays also effective on image, video, text, and also time series data. The latter is the type of data we will focus on in this paper. Specifically, this work is about \textit{time series classification}, a ML task whose goal is to learn a function (i.e a classifier) that maps time series to a set of discrete classes. A time series is an ordered and finite sequence of values. Some examples of time series are  the daily COVID cases and the monthly groundwater level. Time series classification has been applied in several domains including online harassment detection \citep{janiszewski2021time}, medicine \citep{miller2022radar,tan2020data}, emotion recognition \citep{rafique2021deep}, anomaly detection \citep{shen2021time}, and in physics \citep{team2018photometric,leoni2021Fink}. This usability is facilitated by toolkits such as Sktime \cite{loning2019sktime}, which unifies the existing time series classification algorithms under the same user-friendly API. However, the existing methods are generally not applicable to \textit{uncertain} time series; In fact, as far as we know, the uncertain shapelet transform (or simply UST) method \citep{mbouopda2020ust} is the only one that has been designed for uncertain time series classification. 

An uncertain time series (uTS) is a time series of \textit{imprecise} values. Unlike a \textit{regular} time series which is an ordered sequence of real numbers, an uTS is a sequence of pairs of numbers such that the first number of a pair is the best estimate and the second one is the error on that estimate; therefore the exact values of an uTS are unknown.  Figure \ref{fig:sample-chinatown} illustrates a simulated uTS: the blue line is the best estimate and the vertical red bars represent the uncertainty intervals (i.e the exact unknown values are somewhere on the vertical red bars). Figure \ref{fig:sample-plasticc67} is a real uTS extracted from the PLAsTiCC \citep{team2018photometric} dataset: any time series that lies in the red region could be the exact unknown time series. \textcolor{\hlColor}{Uncertain time-series (uTS) classification should not be conflated with conformal learning, as the two address fundamentally different sources and roles of uncertainty. Conformal learning typically assumes that the input data are deterministic and free of measurement noise; its primary objective is to quantify and control the uncertainty of the model’s predictions, for instance by constructing calibrated confidence sets. In contrast, uTS classification seeks to learn a classifier from uncertain data and for data that remain uncertain at prediction time—meaning that uncertainty is intrinsic to the input rather than the output. Despite its practical relevance, uTS classification remains comparatively under-explored, whereas uncertainty estimation more broadly is an active and rapidly evolving field in machine learning \citep{cleaveland2024conformal,cadiz2025uncertainty,nanopoulos2025conformal}. In this work, we directly address the uTS classification problem by developing methods that operate natively on uncertain inputs and explicitly account for their underlying uncertainty.}

\begin{figure}
	\centering
	\begin{subfigure}[c]{.45\linewidth}
		\centering
		\includegraphics[width=.8\linewidth]{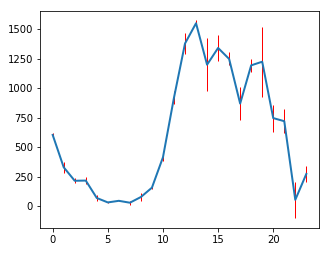}
		\caption{A simulated uTS}
		\label{fig:sample-chinatown}
	\end{subfigure}
	\begin{subfigure}[c]{.45\linewidth}
		\centering
		\includegraphics[width=.8\linewidth]{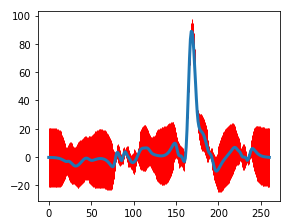}
		\caption{An uTS from PLAsTiCC}
		\label{fig:sample-plasticc67}
	\end{subfigure}
	\caption{Uncertain time series illustrations. \textcolor{\hlColor}{The x-axis is the timestamp while the y-axis represents the value of the time series at the corresponding timestamp. The red bars represent the uncertainty of the measured value.}}\label{fig:uts-ullustration}
\end{figure}

Uncertain time series are preponderant in transient astrophysics. Astronomical objects whose brightness vary with time (a.k.a transients) are primarily characterized by the presence or absence of specific chemical elements found in their spectra. This data taking process (called spectroscopy) is very time-consuming and requires very good observation conditions to be performed. Moreover, since transients are objects which appear in the sky for a limited period of time then disappear forever, there is a small time-window of opportunities when such measurements can be taken. 

Alternatively, we can also associate different classes of astronomical transients to the respective shape of their light curves (brightness variation as a function of time). In this case, we need to repeatedly measure the brightness of the source in a relatively broad region of the wavelength spectrum. This process, called photometry, is less expensive and imposes more manageable constraints on observation conditions. However, measurements are more prone to uncertainties (due to moonlight, twilight, clouds, etc) in the flux determination and the distinction between light curves from different classes is subtle, resulting in less accurate classifications. 
Nevertheless, since there is not enough spectroscopic resources to provide definite label for all photometric observed objects, being able to effectively analyze uncertain photometric light curves means that a wider range of the universe can be quickly understood and at a lower cost.


The Vera C. Rubin Observatory\footnote{\url{https://lsst.org/}} is a ground-based observatory, currently under construction in Chile, whose goal is to conduct the 10-year Legacy Survey of Space and Time (LSST) in order to produce the deepest and widest images of the universe. Equipped with a 3.2 gigapixels camera, the telescope will observe 32 billion stars and galaxies, and is expected to produce about 20 terabytes of data per night. The observatory is expected to start producing data in early 2024, and in order to prepare the community for the arrival of its data, one important data challenge was put in place: the Photometric LSST Astronomical Time-Series Classification Challenge or simply PLAsTiCC \citep{team2018photometric}. The goal is to identify 
machine learning models able to classify $14$ types of transients in simulated data, represented by uncertain time series, or 
light curves. The ultimate goal behind the challenge was to understand which methods are expected to perform  better in LSST-like data, thus preparing the community to the arrival of its data and 
help understanding the universe's expansion history. 
Therefore, using interpretable approaches is very important. However,  existing contributions focused on minimizing the classification loss by employing techniques such as mixture of classifiers and data augmentation \citep{hlovzek2020results} while neglecting explainability. In this paper, we address this problem with explainability in mind. 

We consider two approaches to classify uTS in an explainable manner :  the first one ignores uncertainty and uses only the best estimates, while the second one takes uncertainty into account. Ignoring uncertainty makes the task a \textit{regular} time series classification task, allowing the usage of Shapelet Transform Classification or simply STC \citep{Hills2014stc}, an effective and explainable \textit{regular} time series classification algorithm. This model failed to find any valid shapelet on PLAsTiCC, and therefore could not perform the classification task. We performed extensive hyper-parameter tuning tests, but the result was the same. We also tried to take uncertainty into account by using the Uncertain Shapelet Transform algorithm \citep{mbouopda2020ust}, but as expected, this method also failed since it is an extension of STC for uncertain time series.

In this paper, we propose the Uncertain Scalable and Accurate Subsequence Transform (or uSAST for short) method which is able to achieve an F1-score of $70\%$  while providing faithful explanation similarly to STC. Moreover, this paper is the first of its kind to present an open and reproducible experiment on a real uncertain time series dataset. The rest of this paper is organized as follows: we start by presenting the background and the related works. We continue by describing the uSAST method. Finally, we detail our experiments and the obtained results before concluding this work.

\section{Background}\label{sec:backgound}

\begin{definition}[Time series] A time series (TS) of length $m$ is a finite sequence of ordered values.
	\[
	T=(t_1, t_2,.., t_m), t_i \in \mathbb{R}, m>0
	\]
\end{definition}

\begin{definition}[Uncertain time series] An uncertain time series (uTS) is defined similarly to a time series, but each value has an uncertainty represented by a positive real number.
	\[
	T=(t_1 \pm \delta t_1, t_2 \pm \delta t_2,.., t_m \pm \delta t_m), t_i \in \mathbb{R}, m>0, \delta t_i \in \mathbb{R}_+
	\]
\end{definition}

\begin{definition}[Subsequence] A subsequence (respectively an uncertain subsequence) is a sequence of consecutive values extracted from a TS (respectively an uTS).
\end{definition}

\begin{definition}[Distance] \label{def:distance}The distance between a subsequence $S$ of length $l$ and a time series of length $m$ is defined as follows:
	\begin{equation}
		Dist(S, T) = \min_{P\in T^l} dist(S, P)
	\end{equation}
	where,
	\begin{equation}
		T^l=\{(t_{i}, t_{i+1},...,t_{i+l}) \vert \; 1 \le i \le m-l+1\} \nonumber
	\end{equation}
\end{definition}
The $dist(\cdot,\cdot)$ function in Definition \ref{def:distance} could be any distance metric. In practice the Euclidean Distance (ED) and the Dynamic Time Warping (DTW) are generally used.  The definition is also applicable between uTS and uncertain subsequence by ignoring the uncertainty or by taking it into account using an uncertain distance, the UED distance \citep{mbouopda2020ust}. 

\begin{definition}[Uncertain Euclidean Distance] The Uncertain Euclidean Distance (UED) between two uncertain subsequences $S_1$ and $S_2$ of same length $l$ is defined as:
	\begin{equation}
		UED(S_1, S_2) = \sum_{i=1}^{l}(s_{1,i}-s_{2,i})^2 \pm 2\sum_{i=1}^{l}|s_{1,i}-s_{2,i}| (\delta{s_{1,i}} + \delta{s_{2,i}})
	\end{equation}
\end{definition}

Let $D=\{(T_i, c_i) \vert 1 \le i \le n\}$ be a dataset of $n$ time series $T_i$ (repectively uncertain time series)  with their class labels $c_i$ taken from a discrete finite set $C$ such that the cardinality of $C$ is much less than $n$. We can define the notions of \textit{separator} and \textit{shapelet} for this dataset.

\begin{definition}[Separator] A separator (respectively uncertain separator) is a pair of a subsequence $S$ (respectively uncertain subsequence) and a threshold $\epsilon$ that divide the dataset in two groups $D_{left}$ and $D_{right}$ such that:
	\begin{align}
		D_{left} &= \{(T_i, c_i) \vert Dist(S, T_i) < \epsilon, 1 \le i \le n\} \nonumber\\
		D_{right} &= \{(T_i, c_i) \vert Dist(S, T_i) \ge \epsilon, 1 \le i \le n\}
	\end{align}
\end{definition}

\begin{definition}[Shapelet] A shapelet (respectively uncertain shapelet) is a separator (respectively uncertain separator) that maximizes the information gain similarly to splitting nodes in decision trees \citep{Ye2009shapelet}.
\end{definition}

\section{Related works}\label{sec:related-work}
Time series classification is performed regarding \textit{global} features, \textit{local} features, or both. Historically, only global features were considered; in particular, the classification was done using  the one nearest neighbor (1-NN)  classifier and the DTW distance. The Elastic Ensemble (EE) is an improvement of the \textit{global} features classification, obtained by ensembling several distance measures \citep{Lines2018HC}. The Fast Ensemble of Elastic Distances (FastEE) significantly reduces the computation time of the Elastic Ensemble \citep{Tan2020FastEE}.

\textit{Local} feature-based methods are organized as dictionary-based, interval-based or subsequence-based. Dictionary-based methods proceeds by representing each time series using a finite set of discrete symbols using techniques such as Symbolic Fourier Approximation (SFA) \citep{schafer2012SFA} and Symbolic Aggregate approXimation (SAX) \citep{Lin2007SAX}.  Some methods that implement these techniques are BOSS \citep{schafer2015BOSS}, MUSE \citep{schafer2017muse} and TDE \citep{Middlehurst2020TDE}. Interval-based methods assume that the whole time series is not relevant for classification, but that only some segments (i.e intervals) contain the discriminative features. The first step in these methods is the identification of the relevant intervals, then a set of features (mean, median, slope, ...) are computed for each interval and finally a supervised classifier is trained on the computed features. Some methods that use this approach are TSF \citep{Deng2013TSF}, CIF \citep{Middlehurst2020CIF} and STSF \citep{Cabello2020STSF}. Subsequence-based methods also assume that only some segments are relevant for classification, but unlike interval-based methods which use phase-dependent features (i.e. the locations of the intervals are fixed), subsequence-based methods use phase-independent features. Subsequence-based methods perform in three steps: first, the relevant subsequences are identified, then each time series is transformed to a vector of its distances to the relevant subsequences, and finally a supervised classifier is trained on the obtained vectors. Some of these methods are Shapelet-based decision trees \citep{Ye2009shapelet}, Shapelet Transform \citep{Hills2014stc}, Uncertain Shapelet Transform \citep{mbouopda2020ust} and SAST \citep{mbouopda2021sast}. Uncertain Shapelet Transform (UST) is, to our knowledge, the only subsequence-based method that supports uTS classification. UST uses an uncertain similarity measure to compare uTS. \citet{mbouopda2020ust} showed that using the Uncertain Euclidean Distance (UED) leads to better classification than existing measures such as FOTS \citep{fotso2020frobenius} and DUST \citep{sarangi2010dust}. This is explained by the fact that UED is able to propagate uncertainty, making the whole classification process uncertainty-aware.

The most accurate methods for time series classification are TS-CHIEF \citep{Shifaz2020chief}, which combines both local and global features in a tree-based classifier, HIVE-COTE \citep{Lines2018HC,Middlehurst2021HC2}, which combines several methods to extract local and global features, and ROCKET \citep{Dempster2020ROCKET,Dempster2021MiniRocket}, which employs random convolutional kernels to extract local features. These methods are also known to have explainability issues.
Subsequence-based methods are the most easier to explain; However, all of them failed on the PLAsTiCC dataset. In the next section, we will describe an accurate and interpretable method which produces competitive results on this data set and can also be easily applied to any other uTS dataset.

\section{Uncertain Subsequence Transform Classification}\label{sec:usast}
In this section, we describe a new uncertain time series classification method based on uncertainty propagation as in UST \citep{mbouopda2020ust} and subsequence transform as in SAST \citep{mbouopda2021sast}. In fact, uncertainty propagation is an effective approach to analyze uncertain data \citep{Gruber2020DWT,Liu2021copula} and particularly uncertain time series \citep{mbouopda2020ust}. By using a single random instance from each class, SAST is more scalable and at least as accurate as STC \citep{mbouopda2021sast} while keeping STC interpretability capabilities.


Given a time series dataset, SAST follows four steps: \textit{i)}, one instance is randomly selected from each class: these are called reference time series; \textit{ii)} a set containing every subsequences from the selected time series is created; \textit{iii)} each instance in the dataset is replaced by the vector of its distances to each subsequence obtained in the second step; \textit{iv)} a supervised classifier is trained on the transform dataset.

Performing classification following the SAST steps could be inefficient because of the redundancy in the set of subsequences obtained at the second step. The redundancy is particularly high for small length subsequences and in datasets such as electrocardiogram (ECG) and PLAsTiCC, in which repetitive patterns occur very often. Furthermore, the third step is based on the application of Definition \ref{def:distance} using the Euclidean distance 
and, therefore, only the most similar subsequence 
is considered; however, taking into account the number of occurrences of the best match is important in some contexts. To overcome these limitations, we define the notion of \textit{$\epsilon$-similarity} as follows:

\begin{definition}[$\epsilon$-similarity]\label{def:epsilon-similarity} Two subsequences (respectively uncertain subsequences) $S_1$ and $S_2$ of same length $l$ are $\epsilon\text{-similar}$ if the distance between them is less than or equal to a user-defined threshold $\epsilon \ge 0$.
	
	\begin{equation}
		\epsilon\text{-similar}(S_1, S_2) = \begin{cases}
			True, & \text{if } dist(S_1,S_2) \le \epsilon\\
			False, & \text{otherwise}
		\end{cases}
	\end{equation}
\end{definition}

\begin{theorem}\label{thm}
	The $\epsilon\text{-similar}$ relationship is not transitive.
\end{theorem}

\begin{pot}
	Let $X$, $Y$, and $Z$ be three subsequences of same length $l$ such that $\epsilon\text{-similar}(X,Y)=True$ and $\epsilon\text{-similar}(Y,Z)=True$. Let us assume that the transitivity property is verified, that is $\epsilon\text{-similar}(X,Z)=True$. A counterexample is built by considering $X$, $Y$, and $Z$ as points in a high dimensional space ($\mathbb{R}^l$) such that $dist(X,Y)=dist(Y,Z)=\epsilon$, and $XY \perp XZ$. The following derivation proves the theorem:
	
	\begin{align}
		dist(X,Z) =& \sqrt{dist(X,Y)^2 + dist(Y,Z)^2} \nonumber \\
		=& \sqrt{\epsilon^2 + \epsilon^2} \nonumber \\
		=& \epsilon\sqrt{2} \nonumber \\
		>& \epsilon \nonumber \\
		\implies& \epsilon\text{-similar}(X,Z)=False\\
	\end{align}
\end{pot}

This non-transitivity property of the $\epsilon\text{-similar}$ relationship prevents considering two subsequences to be similar because there is another subsequence similar to both of them. Meaning that the similarity between each pair of subsequences must be computed independently.

Using Definition \ref{def:epsilon-similarity}, we can reduce redundancies and count subsequence frequencies in SAST. The updated SAST method, hereafter 
SAST+, is detailed in Algorithm \ref{alg:sast+}.

\begin{algorithm}[ht]
	\caption{SAST+}
	\label{alg:sast+}
	\begin{algorithmic}[1]
		\Require $D=\{(T_1, c_1), (T_2,c_2), ..., (T_n, c_n)\}$, $k$: the number of instances to use per class, $length\_list$: the list of subsequence lengths, $C$: the classifier to use, $\epsilon$: $\epsilon\text{-similarity}$ parameter.
		
		\Comment{\textcolor{gray}{\small{Randomly select $k$ instances per class from the dataset}}}
		\State $D_c \leftarrow randomlySelectInstancesPerClass(D,k)$ \label{line:random-select}
		
		\Comment{\textcolor{gray}{\small{Generate every patterns of length in $length\_list$ from $D_c$, using $\epsilon$ to remove similar patterns}}}
		\State $S \leftarrow generateSubsequences(D_c, length\_list, \epsilon)$ \label{line:gen-can}
		
		\State $D_f \leftarrow \emptyset$ \label{line:start-transform}
		
		\For{$i \leftarrow 1 \textbf{ to } n $} \Comment{\textcolor{gray}{\small{Transformed the dataset using every patterns in $S$}}}
		\State $x_i \leftarrow []$
		
		\For{$j \leftarrow 1 \textbf{ to } |S| $} \Comment{\textcolor{gray}{\small{The procedure $distAndCount(T_i, S_j, \epsilon)$ returns $Dist(T_i, S_j)$ and the number of occurrences of the subsequence $S_j$ in $T_i$ }}}
		\State $x_i[j], x_i[j+|S|] \leftarrow distAndCount(T_i, S_j, \epsilon)$
		\EndFor
		\State $D_f \leftarrow D_f \cup \{(x_i, c_i)\}$
		\EndFor \label{line:end-transform}
		
		\State $clf \leftarrow trainClassifier(C, D_f)$ \Comment{\textcolor{gray}{\small{Train the classifier on the transformed dataset}}} \label{line:train-classifier}
		
		\State \Return{($clf$, $S$)} \Comment{\textcolor{gray}{\small{The trained classifier and the subsequences}}}
	\end{algorithmic}
\end{algorithm}

The time complexity of the SAST method is $O(N_c) + O(kN_cm^2) + O(nm^3) + O(classifier)$, where $N_c$ is the number of classes, $n$ the number of time series, $m$ the length of the time series and $k$ the number of reference time series per class \citep{mbouopda2021sast}. In practice, it is not necessary to have $k$ greater than one. Removing redundancies in SAST is done only once (during the training phase) with a theoretical time complexity of $O(km^4)$ ; counting frequencies is done while computing the distance in a constant time. Therefore, the SAST+ time complexity is $O(N_c) + O(kN_cm^2) + O(nm^3) + O(classifier) + O(km^4)$ which is asymptotically equivalent to $O(classifier) + O(km^4)$. Removing redundancies makes SAST+ much faster than SAST during inference.

Similarly to the Uncertain Shapelet Transform \citep{mbouopda2020ust}, the uncertain SAST+ (uSAST+) is obtained by using UED as the distance metric in Algorithm \ref{alg:sast+}; allowing uncertainties to be propagated to the classifier which then uses these uncertainties to learn robust decision boundaries. \textcolor{\hlColor}{More precisely, the procedure $distAndCount(T_i,S_j,\epsilon)$ uses uses UED to compute both the uncertain similarity between the uTS and each uncertain subsequence, and the number of subsequence occurrences in the uTS. Because the classifier does not natively handle uncertain inputs, each uncertain quantity is represented by two values: (i) its best estimate (e.g., the mean), and (ii) the associated uncertainty around that estimate. These two numbers and any other features (e.g., frequency) are then passed to the classifier as its input features.}

\section{Experiment}\label{sec:experiment}

\subsection{The PLAsTiCC dataset}
As far as we know, existing methods published on uTS classification have never been evaluated on real uncertain time series datasets, but solely on simulated datasets. The corresponding simulated datasets have never been made publicly accessible neither for reproducibility reasons, nor for facilitating research on uTS. In this work, we evaluate our method on a realistic publicly available uncertain time series dataset from the astrophysics domain. 

The Photometric LSST Astronomical Time-Series Classification Challenge (PLAsTiCC) dataset contains uncertain time series representing the brightness evolution of astronomical transients including supernovae, kilonovae, active galactic nuclei and eclipsing binary systems \cite{team2018photometric}, among others. Each object is represented as a multivariate uncertain time series of $6$ dimensions named \textit{u, g, r, i, z, y}, each corresponding to a particular broadband wavelength filter. After the challenge was finished, the organizers made available an updated version of the data through Zenodo\footnote{\url{https://zenodo.org/record/2539456}} with some bug fixes and the classification answers for both the training and test sets. In this work, we demonstrate our method using only uncertain time series from the training set, but the methodology is general enough to be extended to the test set. There are $7848$ transients in the dataset, grouped in $15$ different classes (14 types of transients and one additional class that contains any other type of transient) identified by numbers as shown in Table \ref{tab:class-names}. The number of objects in the classes are highly imbalanced. More specifically, the most underpopulated class has only $0.3\%$ of objects, whereas the most populated one contains $29\%$ of the objects. Furthermore, the dataset contains a lot of missing observations. We handled this with the help of astrophysicists who suggested to fill missing data using a rolling average with a window of length $5$. Missing values and corresponding error bars are replaced by the mean and standard deviation of the window. This procedure translated the original dataset into a homogeneously sampled uncertain time series. The preprocessed dataset is made public
\footnote{Cleaned dataset: \url{https://drive.uca.fr/f/f0741be3fb77402f8e82/}}.

\begin{table}[htbp]
	\centering
	\caption{Class names and their identifiers in PLAsTiCC }
	\label{tab:class-names}
	\begin{tabular}{c|c}
		\toprule
		Class name & Identifier \\
		\midrule
		Point source $\mu$-lensing & 6 \\
		Tidal disruption event (TDE) & 15 \\
		Eclipsing binary event (EBE) & 16 \\
		Core-collapse supernova Type II (SNII) & 42 \\
		Supernova Type Ia-x (SNIax) & 52\\
		Mira Variable & 53 \\
		Core-collapse Supernova Type Ibc (SNIbc) & 62 \\
		Kilonova (KN) & 64 \\
		M-dwarf & 65 \\
		Supernova Type Ia-91bg (SNIa-91bg) & 67 \\
		Active galactic nucleus (AGN) & 88 \\
		Supernova Type Ia (SNIa) & 90 \\
		RR Lyrae & 92 \\
		Super Luminous Supernova (SLSN) & 95 \\
		‘Other’ class & 99 \\
		\bottomrule
	\end{tabular}
\end{table}

Our implementation uses the Python programming language and is based on the Scikit-learn machine learning library \citep{scikitlearn} and the Sktime time series dedicated machine learning library \citep{loning2019sktime}. The experiment is run on a computing node equipped with 1 Gb of RAM and an AMD EPIC 7452 processor containing 64 logical cores of 2.35 GHz frequency. The source code of our experiments and all the results we discuss in this paper are publicly available on GitHub\footnote{Source code: \url{https://github.com/frankl1/usast}}.

\subsection{Results}
Since PLAsTiCC is a multivariate uncertain time series dataset, the subsequence transformation is performed on each dimension independently. The transformations from each dimension are then concatenated together to build a large matrix which is subsequently  fed to the supervised classifier. We used $80\%$ of the data 
for training and the remaining is used for testing.

\subsubsection{Shapelet-based methods results:}
Shapelet-based classification is a special case of subsequence-based classification which consider only shapelets as relevant subsequences. We considered two shapelet-based methods STC \citep{Hills2014stc} and UST \citep{mbouopda2020ust} for their interpretability. For both methods, we kept every parameters to their default values except the \textit{minimum information gain} parameter which is the threshold used to decide if a separator is a valid shapelet. We tried different values for this parameter without success, none of these methods were able to find a single valid shapelet in the dataset. Since feature extraction was not successful, classification was not possible. This result is due to the dataset being highly imbalanced and the uncertain time series from different classes being too similar in shape. The same dimension of two randomly selected samples from two different classes is shown on Figure \ref{fig:two-classes-illustration}. The left figure which is a Supernova Type Ia-x (SNIax) looks like a left-shifted version of the right figure which is a Supernova Type Ia-91bg (SNIa-91bg). SNIax and SNIa-91bg are known to be difficult to distinguish by astrophysicists. This observation holds, with different magnitude, for other classes in the PLAsTiCC dataset and therefore, any shapelet-based methods might struggle to find shapelets in this dataset.

\begin{figure}[htbp]
	\centering
	\begin{subfigure}{.45\linewidth}
		\centering
		\includegraphics[width=.8\linewidth]{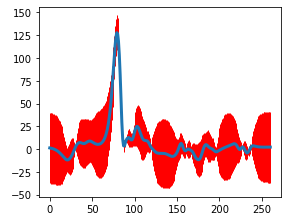}
		\caption{Supernova Type Ia-x}
	\end{subfigure}
	\begin{subfigure}{.45\linewidth}
		\centering
		\includegraphics[width=.8\linewidth]{plasticc_g_sample_class67.png}
		\caption{Supernova Type Ia-91bg}
	\end{subfigure}
	\caption{Two supernova from PLAsTiCC. They look similar in terms of shapes although they are from distinct classes.}
	\label{fig:two-classes-illustration}
\end{figure}

\subsubsection{SAST-based methods results:} For this experiment we considered different SAST+ configurations in order to measure the effect of taking uncertainty into account, dropping duplicates and counting the number of occurrences of patterns (i.e patterns frequency). We named configurations that ignore uncertainty as SAST<X> and those which take uncertainty into account as uSAST<X>, where <X> is either : \textit{i)} an empty string to specify that duplicate subsequences are not removed and the patterns frequency is ignored; \textit{ii)} the character \textit{d}, meaning that duplicate patterns are removed; \textit{iii)} the string \textit{dc}, meaning that duplicate patterns are removed and the frequency of patterns is taken into account.

We use three different supervised classifiers, namely Random Forest (RF), eXtreme Gradient Boosting (XGBoost) and the Ridge regression with Leave-One-Out cross-validation (RidgeCV). The cross-validation procedure is used to find the best regularization parameter. We set the minimum and maximum subsequence lengths to $20$ and $60$ respectively, with a step of $10$. Compared to a step of $1$, a step of $10$ reduces the chance of having similar subsequences while reducing the number of subsequences to be used. We observed that the classification performance is better with this setup as can be seen in the supplementary material. \textcolor{\hlColor}{The $\epsilon$-similarity is computed using $\epsilon=0.25$. This value was selected through a grid search over a set of predefined candidates. Our experiments showed that values above $0.5$ discard a substantial number of relevant subsequences and consequently reduce classification accuracy, whereas values below $0.5$ fail to prune enough subsequences to produce any meaningful reduction in computation time.} The parameters of the classifiers are left to their default values, except for the regularization parameter in RidgeCV which is selected using cross-validation. \textcolor{\hlColor}{As \citet{mbouopda2021sast} demonstrated that additional reference time series per class offer limited accuracy improvement at the cost of higher computational overhead, we follow their setting and use $k=1$. We have also considered every possible subsequence length starting from 3 to the time series length, letting the classifier automatically select the most relevant ones. This is computationally feasible because we are retrieving the subsequences from the reference time series and not from the full dataset}. As the reference time series are chosen randomly, we run each experiment $3$ times and we report the average precision, recall, F1 score, cross entropy loss and the time taken for training and inference (in hours). \textcolor{\hlColor}{We tested multiple alternative sampling strategies—including class prototypes and sampling several instances per class—but our experiments showed no accuracy or runtime gains relative to the simple strategy proposed in the original SAST paper}. As PLAsTiCC is an imbalanced multiclass dataset, we use a weighted average to compute the precision, recall and F1 score; the weights being the percentage of each class in the dataset. Table \ref{tab:usast-results} shows the result using the XGBoost classifier only as it has led to the best classification performance. However, detailed results are available in the supplementary material.

\begin{table}[htbt]
	\centering
	\caption{Results on PLAsTiCC averaged over $3$ runs.}\label{tab:usast-results}
	\begin{tabular}{l|c|c|c|c|c}
		\toprule
		& Precision & Recall & F1 score & LogLoss & Time (h) \\
		\midrule
		uSAST & $ 0.72 \pm 0.01 $ & $ 0.72 \pm 0.00 $ & $ 0.69 \pm 0.01 $ & $ 0.96 \pm 0.01 $ & $ 51.03\pm 0.12 $ \\
		\midrule
		\textbf{uSASTd} & $ \mathbf{0.72 \pm 0.00} $ & $ \mathbf{0.73 \pm 0.00} $ & $ \mathbf{0.70 \pm 0.01} $ & $ \mathbf{0.97 \pm 0.01} $ & $ \mathbf{43.49 \pm 0.27} $ \\
		\midrule
		uSASTdc & $ 0.71 \pm 0.01 $ & $ 0.72 \pm 0.01 $ & $ 0.69 \pm 0.01 $ & $ 0.96 \pm 0.01 $ & $ 43.52 \pm 0.72 $ \\
		\bottomrule
	\end{tabular}
\end{table}

The first observation is that any variant of our proposed method is able to achieve around $70\%$ precision, recall and F1 score, unlike shapelet-based methods which completely failed on the PLAsTiCC dataset. This result corroborates with the claim that pruning subsequences before the effective classification could sometimes lead to poor performance \citep{mbouopda2021sast}. Dropping duplicates, counting patterns frequency or doing both does not have significant impact on the classification performance. However, dropping duplicate makes the models faster. In particular, uSASTd is about $12$ hours faster than uSAST. Counting pattern frequency does not add a computation overhead because it is done while computing the distance in $O(1)$ time. 

Choosing the right subsequence lengths to considered is challenging and assessing all possible values is computationally expensive; However, domain knowledge could guide in setting this parameter as it is application-dependent.


PLAsTiCC contains objects that are either galactic or extra-galactic, and whose light curves were obtained following a 
Deep Drilling Fields (DDF) or Wide Fast Deep (WFD) observation strategy. Extra-galactic objects are further away than galactic ones, they are fainter and more difficult to be observed. DDF light curves 
contain more frequent observation points than WFD ones. Thus, DDF light curves provide a more certain determination of the time series properties than their WFD counterparts which have more uncertainties.  
Table \ref{tab:usast-results-grouped} gives the performances of the model uSASTd regarding if the objects are galactic or not, DDF or WFD. The model is considerably better at classifying galactic objects than extra-galactic ones, and a little better at classifying DDF objects than WFD ones. While the model achieves an F1 score of $94\%$ for galactic objects in DDF, it achieves an F1 score of only $67\%$ for extra-galactic objects in WFD. This is directly related to the astrophysical nature of galactic objects. These are, in general, variables whose brightness go through many cycles within the 3 years covered by our data. On the other hand, extragalactic objects are dominated by transients, consisting of only 1 region of signal which never repeats, thus rendering a smaller quantity of information encoded in its time series.

\begin{table}[ht]
	\centering
	\caption{uSASTd performance regarding if the object are galactic or extra-galactic, are from the DDF or WFD.}\label{tab:usast-results-grouped}
	\begin{tabular}{l|l|ccc}
		\toprule
		&          &  Galactic &  Extra-galactic &  Both \\
		\midrule
		& Precision &      0.96 &          0.73 &  0.77 \\
		DDF & Recall &      0.94 &          0.75 &  0.79 \\
		& F1 score &      0.94 &          0.71 &  0.76 \\
		\midrule
		& Precision &      0.94 &          0.67 &  0.71 \\
		WDF & Recall &      0.84 &          0.64 &  0.71 \\
		& F1 score &      0.87 &          0.61 &  0.67 \\
		\midrule
		& Precision &      0.94 &          0.68 &  0.72 \\
		Both & Recall &      0.86 &          0.67 &  0.73 \\
		& F1 score &      0.88 &          0.64 &  0.70 \\
		\bottomrule
	\end{tabular}
\end{table}

The data set includes 6 classes with overall similar behavior (42, 52, 62, 67, 90, 95). Among these, astronomers are specially interested in type 90 (SNIa), which is used as distance indicator in cosmological analysis \cite{Ishida2019}. Reporting our results as a binary problem with class 90 against all others, we achieve $85\%$ precision, $81\%$ recall and $82\%$ F1 score. Therefore, our method is able to correctly classify a high proportion of SNIa despite its similar behavior to other classes.

\subsubsection{Ablation study:} Here, we study the impact of taking uncertainty into account. In particular, we compare the results obtained when uncertainty is ignored (Table \ref{tab:sast-results}) to the results obtained when uncertainty is taken into account (Table \ref{tab:usast-results}). 

\begin{table}[htbp]
	\centering
	\caption{Results on PLAsTiCC averaged over $3$ runs when uncertainty is ignored.}\label{tab:sast-results}
	\begin{tabular}{l|c|c|c|c|c}
		\toprule
		& Precision & Recall & F1 score & LogLoss & Time (h) \\
		\midrule
		SAST & $ 0.65 \pm 0.01 $ & $ 0.67 \pm 0.00 $ & $ 0.63 \pm 0.00 $ & $ 1.16 \pm 0.01 $ & $ 16.41 \pm 0.52 $ \\
		\midrule
		SASTd & $ 0.66 \pm 0.02 $ & $ 0.68 \pm 0.00 $ & $ 0.64 \pm 0.00 $ & $ 1.14 \pm 0.00 $ & $ 12.79 \pm 0.84 $\\
		\midrule
		SASTdc & $ 0.66 \pm 0.01 $ & $ 0.68 \pm 0.00 $ & $ 0.64 \pm 0.01 $ & $ 1.14 \pm 0.01 $ & $ 12.99 \pm 0.30 $ \\
		\bottomrule
	\end{tabular}
\end{table}

Taking uncertainty into account increases the classification performance in terms of precision, recall, F1 score and cross entropy loss. In fact, from SASTd to uSASTd, there is a gain of $6\%$ in precision, $5\%$ in recall, $6\%$ in F1 score. It can also be seen that the model is more confident on its predictions as the loss has decreased. However, this gain in performance requires almost four times more computation. 

\subsubsection{Comparison to SOTA:} In this subsection, we compare our proposed method to the state-of-the-art multivariate time series classification methods ROCKET \citep{Dempster2020ROCKET}, MUSE \citep{schafer2017muse} and XEM \citep{fauvel2022xem} which have been shown to be among the most accurate methods for this task \citep{ruiz2021great}. Results are shown in Table \ref{tab:usast-vs-sota}.


\begin{table}[hbtp]
	\centering
	\caption{uSASTd vs SOTA results.}\label{tab:usast-vs-sota}
	\begin{tabular}{l|c|c|c|c}
		\toprule
		& Precision & Recall & F1 score & Time (h) \\
		\midrule
		uSASTd & $ 0.72 \pm 0.00 $ & $ 0.73 \pm 0.00 $ & $ 0.70 \pm 0.01 $ & $ 43.49 \pm 0.27 $ \\
		\midrule
		MUSE & $ 0.71 \pm 0.01 $ & $ 0.73 \pm 0.01 $ & $ 0.71 \pm 0.01 $ & $ 3.36 \pm 0.04 $\\
		\midrule
		ROCKET & $0.77 \pm 0.00 $ & $0.77 \pm 0.00$ & $0.75 \pm 0.00$ & $0.05 \pm 00$\\
		\midrule
		XEM & $0.69 \pm 0.01$ & $0.71 \pm 0.00$ & $0.69 \pm 0.00$ & $12.24 \pm 0.46$ \\
		\bottomrule
	\end{tabular}
\end{table}

The classification performance of our method is comparable to those of the SOTA methods. In particular, uSASTd achieves better precision, recall and F1 score compared to XEM on PLAsTiCC. uSASTd and MUSE have similar classification performance. ROCKET achieves the best classification performance. SOTA methods are faster than our proposal. Except for XEM which is explainable-by-design, SOTA methods are not explainable. In fact, ROCKET uses the proportion of positive values obtained after applying random convolutions. MUSE uses bag of words obtained after applying some transformations to the time series. These features have no particular meaning for domain experts. Our method does not have this limitation, as it is based on features that are intelligible to domain experts.

\subsubsection{Explainability:}One of the best properties of subsequence-based classification is its interpretability. 
The explanation could be done either locally, when it concerns only a single instance, or globally when it concerns the whole model. In any case, this is generally done by inspecting the model in order to extract the most discriminative subsequences \citep{Ye2009shapelet}. These subsequences could also be found using a post-hoc method such as LIME \citep{ribeiro2016should} or SHAP \citep{lundberg@2017}, but since our approach is explainable-by-design, inspecting the model is sufficient. More specifically, since the classifier used in our model is tree-based, the information gain can be used as a measure of the discriminative power of the subsequences similarly to what is done in shapelet-based methods. The local explainability of our method is obtained by inspecting the subsequence on  which the model focused the most in order to make the prediction for a single instance. Figure \ref{fig:local-interpretability} shows local explanations for a Supernova Type Ia (SNIa) and a Core-collapse Supernova Type II-P (SNII-P) correctly classified by the model. The ``P'' in the denomination of the latter references the plateau phase observed in its time-series just after maximum brightness. This feature is clearly shown in the bottom panel of Figure \ref{fig:local-interpretability}. This confirms that our model focuses on the relevant regions and dimensions of the time series to make the classification. 

\begin{figure}[htbp]
	\centering
	\includegraphics[width=\linewidth]{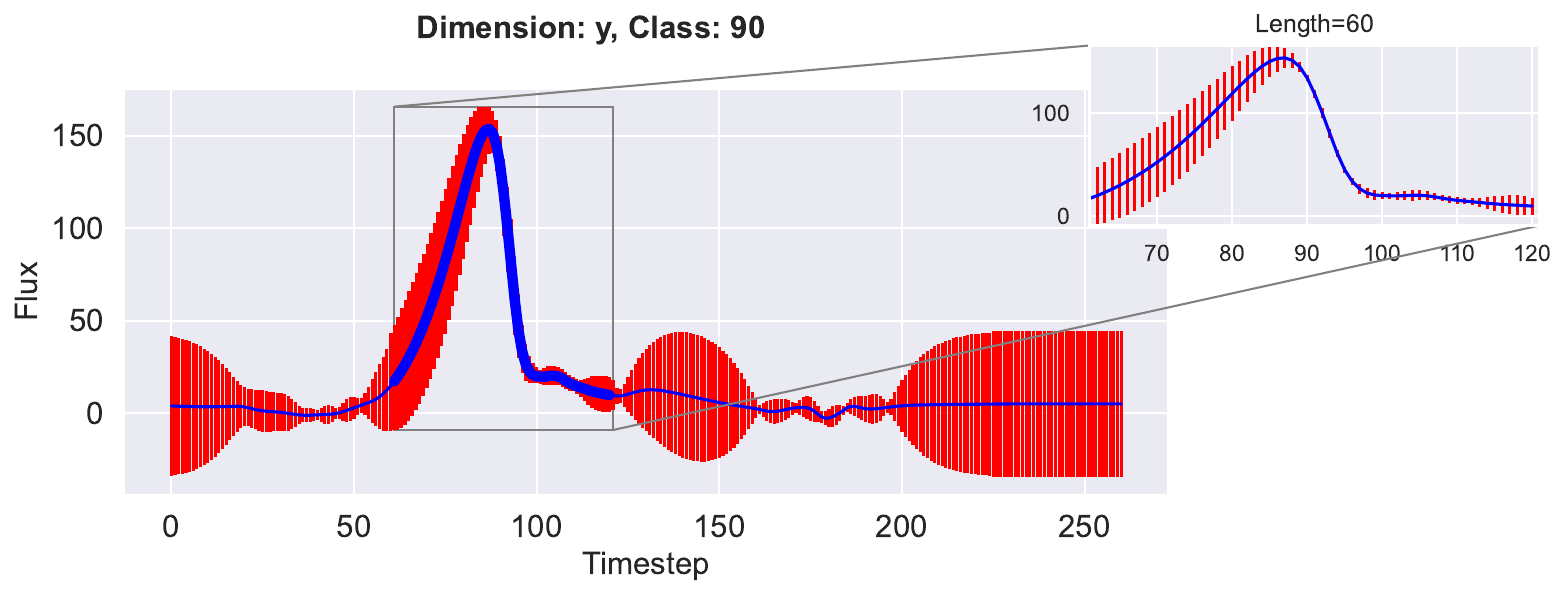}
	\includegraphics[width=\linewidth]{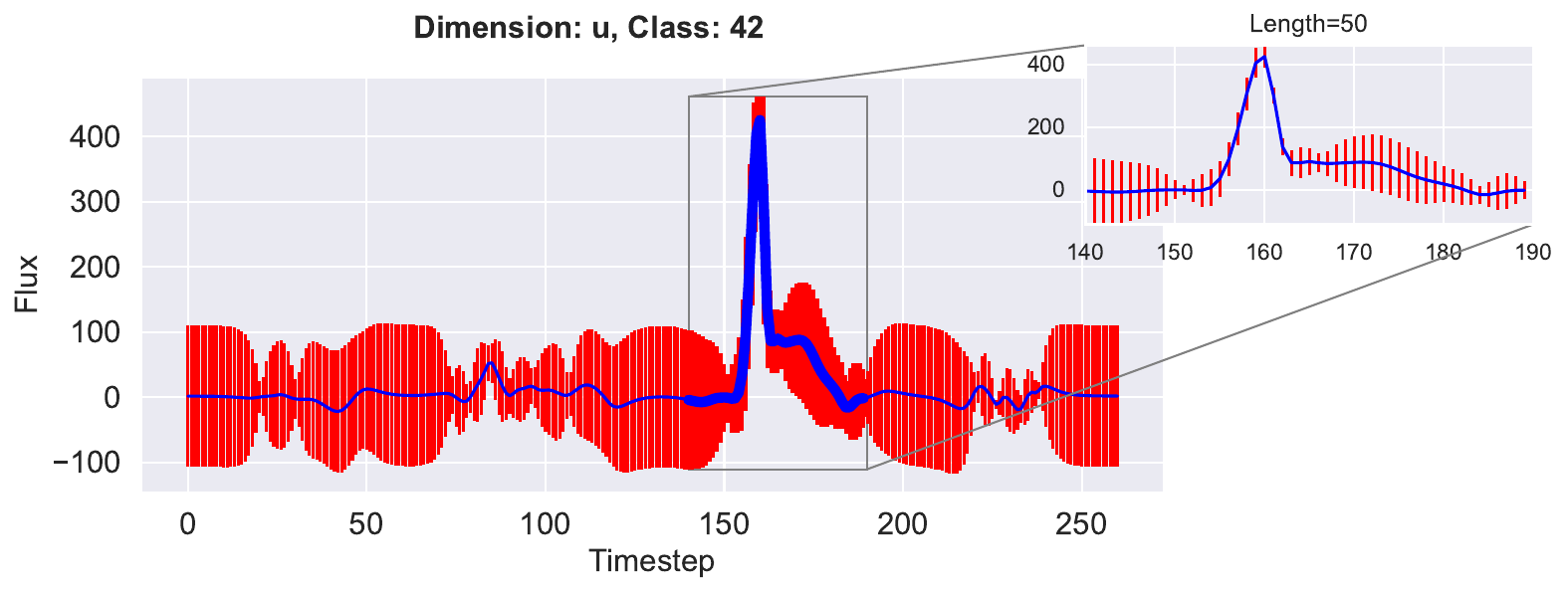}
	\caption{Local explainability of a Supernova Type Ia (top) and a Core-collapse Supernova Type II (bottom). \textcolor{\hlColor}{The y-axis shows the light intensity of the object while the x-axis represent the timestamp}.}.
	\label{fig:local-interpretability}
\end{figure}

Being able to correctly learn the dimension's relevance is crucial as the discriminative subsequence may appear only in a subset of the dimensions. Furthermore, the location of the discriminative subsequence may not be the same on every dimension. In PLAsTiCC in fact, depending how far is the object, the light may be visible only on some wavelengths (i.e. dimension). Due to the accelerated expansion of the universe, objects which are further away are also moving with a higher velocity. Thus, there is a Doppler effect in the observed light which shifts it to higher wavelengths. Thus, closer (galactic) objects will generally have higher signals in lower wavelengths than further away (extragalactic) ones. Our method perfectly captures the Doppler effect unlike XEM which cannot identify from which dimensions the discriminative subsequences is located. For instance, XEM found the most discriminative subsequence to be in the region starting at time step 136 and ending at 187 for an instance in class 62 which correspond to the class of Core-Collapse Supernova Type Ibc (SNIbc). This region is illustrated on Figure \ref{fig:xem-interpretability} for each of the six dimensions. Visualization confirms that this region contains the most important subsequences, which are located on dimensions \textit{i, r, y}, and \textit{z}. Additionally, it is observed that this region contains no signal on dimensions \textit{g} and \textit{u}, meaning that these dimensions are not relevant to classify this instance. It might be alright to make use of visualization to find out which dimensions really contain the important subsequence for a single sample, but this is infeasible for many instances, especially when the number of dimensions is high. uSAST handles this by automatically extracting the most important features and their corresponding dimensions.

\begin{figure}[hbtp]
	\centering
	\includegraphics[width=\linewidth]{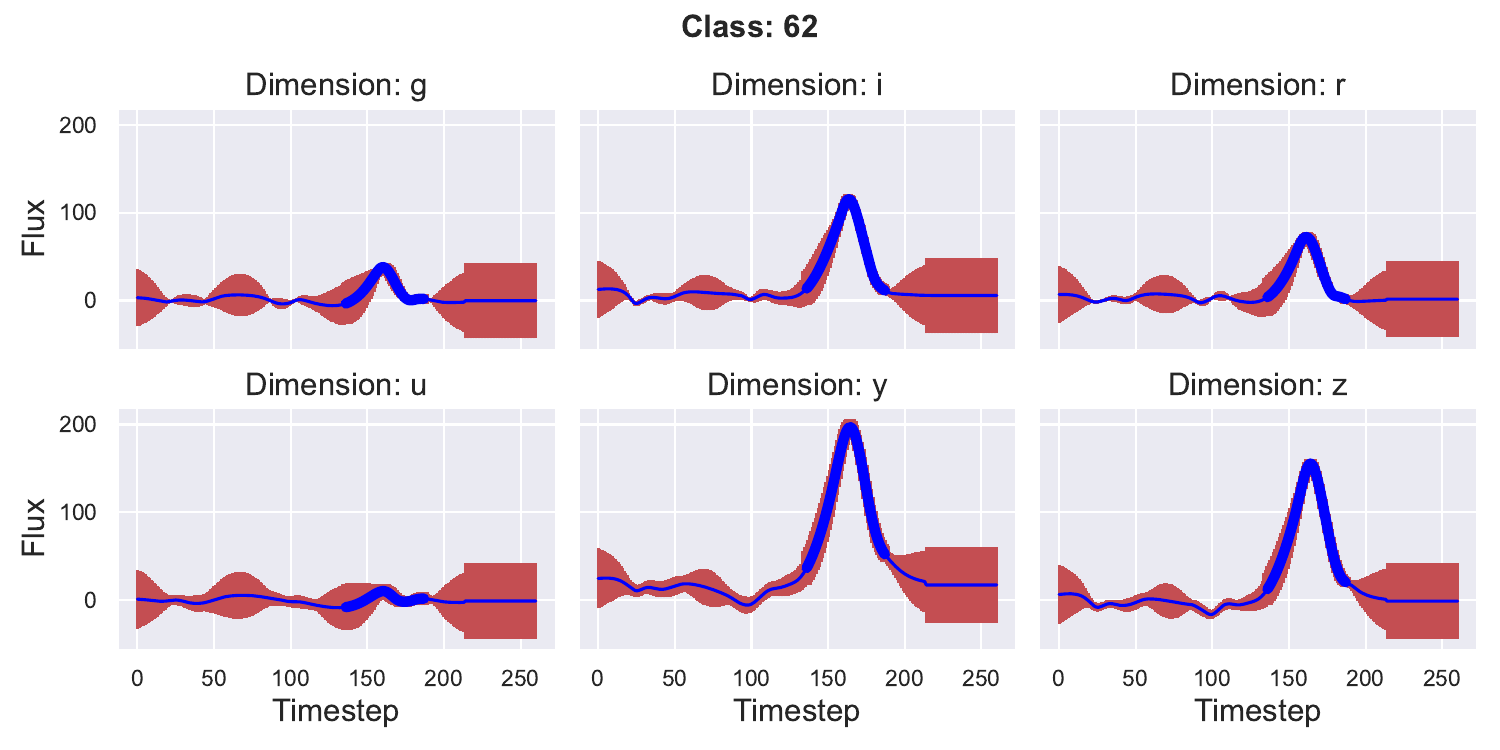}
	\caption{Local explainability of a Core-collapse Supernova Type Ibc by XEM. \textcolor{\hlColor}{The y-axis shows the light intensity of the object while the x-axis represent the timestamp}.}.
	\label{fig:xem-interpretability}
\end{figure}

A global explanation is obtained by building a subsequence-based profile of each of the class.  The top $20$ most discriminative subsequences from the uSASTd model are shown in Figure \ref{fig:top-plasticc-subsequences}. Subsequences that are from the same class label are plotted with the same color, 
its rank, its class label and its type are given at the top of its corresponding plot. The type is either \textit{Value} if the discriminative power comes from the value itself or \textit{Uncertainty} if the discriminative powers comes from the uncertainty. The dimension from which the subsequences are coming from are also given on the figure.

\begin{figure}[htbp]
	\centering
	\includegraphics[width=\linewidth]{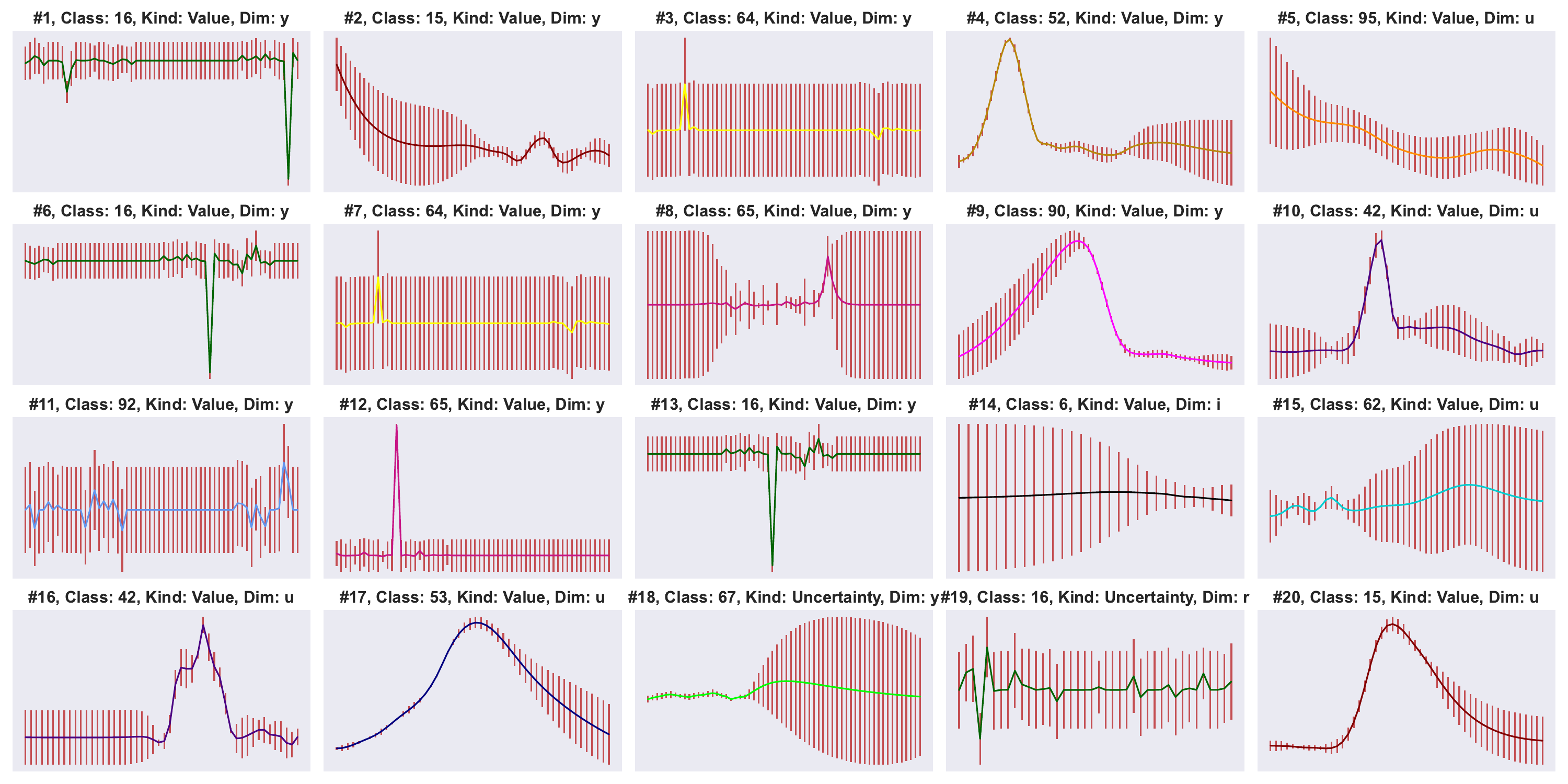}
	\caption{The top $10$ most discriminative subsequences in the PLAsTiCC dataset. \textcolor{\hlColor}{The y-axis shows the light intensity of the object while the x-axis represent the timestamp}.}\label{fig:top-plasticc-subsequences}
\end{figure}

It is observed that the discriminative power is generally due to the value, but sometimes it is due to the uncertainty (for example subsequences \#18 and \#19). Seeing that some subsequences are important because of their uncertainty emphasizes the fact that taking uncertainty into account is important and improves the classification performance. There are also some subsequences that are too similar despite the fact that duplicate subsequences have been dropped; for instance, the subsequences \#3 and \#7. This is because the similarity between subsequences is computed using the Uncertain Euclidean Distance (UED) which considers the subsequences to be perfectly aligned. This problem can be resolved by using an elastic distance such as the DTW distance at the cost of more computational time since such distances generally have at least quadratic time complexity while UED is linear.
From the domain knowledge point of view, these discriminative subsequences are able to grasp the important shapes commonly associated with their respective class of astronomical transients. Subsequences \#1 and \#6 were taken from class 16 (eclipsing binary) and clearly show the expected light curve from a well measured binary system where one star eclipses the other exactly in the line of sight, thus leading to a decrease in brightness. Subsequence \#19 is also associated to the eclipsing binary class, but in this case the signal is less clear, corresponding to an object which is further away -- thus leading to low signal and large uncertainties. We also call attention to the supernova-like behavior exhibited by subsequences \#4 and \#9 -- one single burst events whose brightness are only visible for weeks to months. The fact that such characteristic behaviors are easily spotted in the list of most important subsequences certifies that our final classification results are in line with the expert definition of such classes and hence, shows that our model is safe and trustworthy. Moreover, further investigations of a more extensive list of important subsequences have the potential to reveal unexpected time series shapes and promote the development of more detail theoretical models for such astrophysical sources.

\section{Conclusion and future directions}\label{sec:conclusion}
The classification of time series with available uncertainty measures is an under-explored and challenging task. In this work, we proposed an approach to perform this task with a global F1 score of $70\%$, without using techniques such as data augmentation nor oversampling. The explainability of the proposed approach allows domain experts to not only understand individual predictions, but also to characterized each class by a set of subsequences with high discriminative power, which can then be used to perform other important tasks in astrophysics such as novel astronomical transients detection and anomaly detection. The ablation study shown the positive impact of taking uncertainty into account. A limitation of the approach is the time complexity, which could be considerably high for datasets with relatively long uncertain time series. A future direction would consist of further reducing the number of subsequences to be used and optimizing the computation time of the method. Another future direction would consist of finding a better way of managing uncertainty during the classification step in order to improve the performances. Nevertheless, the results presented in this work illustrate how our approach is effective in identifying meaningful subsequences which, beyond the classification performance, can provide important information to the expert. The approach is flexible enough to be applied to other scientific domains where uncertain time series are the common, thus enabling future advances in multiple uncertainty-related subject areas. 

\printcredits

\bibliographystyle{cas-model2-names}

\bibliography{bib}%



\end{document}